\documentclass[aos,preprint]{imsart}

\usepackage{natbib}
\usepackage{amsmath}
\usepackage[psamsfonts]{amssymb}
\usepackage{amsthm}
\usepackage[bookmarks=false]{hyperref}
\usepackage{hypernat}
\usepackage{cancel}
\usepackage[dvips]{graphicx}
\usepackage{color}
\usepackage{algorithm}
\usepackage{algorithmic}
\usepackage{subfig}
\usepackage{url}

\setattribute{journal}{name}{}

\startlocaldefs
\usepackage{macros}
\newtheorem{theorem}{Theorem}[section]

\urldef\rpaurl\url{http://www.cs.toronto.edu/~rpa}
\endlocaldefs

\begin{document}
\begin{frontmatter}
  \title{Learning the Structure of Deep Sparse Graphical Models}
  \runtitle{Learning the Structure of Deep Sparse Graphical Models}
  \begin{aug}
    \author{\fnms{Ryan P.} \snm{Adams}\thanksref{t1}\corref{}%
      \ead[label=u1,url]{http://www.cs.toronto.edu/~rpa/}%
      \ead[label=e1]{rpa@cs.toronto.edu}},
    \author{\fnms{Hanna M.} \snm{Wallach}\ead[label=e2]{wallach@cs.umass.edu}}
    \and
    \author{\fnms{Zoubin}
      \snm{Ghahramani}\ead[label=e3]{zoubin@eng.cam.ac.uk}}
    \runauthor{R.P.\ Adams et al.}

    \affiliation{University of Toronto, University of Massachusetts\\
      and University of Cambridge}
    \address{Ryan P. Adams\\
      Department of Computer Science\\
      University of Toronto\\
      10 King's College Road\\
      Toronto, Ontario M5S 3G4, CA\\
      \printead{e1}}
    \address{Hanna M. Wallach\\
      Department of Computer Science\\
      University of Massachusetts Amherst\\
      140 Governors Drive\\
      Amherst, MA 01003, USA\\
      \printead{e2}}
    \address{Zoubin Ghahramani\\
      Department of Engineering\\
      University of Cambridge\\
      Trumpington Street\\
      Cambridge CB2 1PZ, UK\\
      \printead*{e3}}
    
    \thankstext{t1}{\rpaurl}
  \end{aug}
  
  \begin{abstract}
    Deep belief networks are a powerful way to model complex probability
    distributions.  However, learning the structure of a belief network,
    particularly one with hidden units, is difficult. The Indian buffet
    process has been used as a nonparametric Bayesian prior on the directed
    structure of a belief network with a single infinitely wide hidden layer.
    In this paper, we introduce the cascading Indian buffet process (CIBP),
    which provides a nonparametric prior on the structure of a layered,
    directed belief network that is unbounded in both depth and width, yet
    allows tractable inference.  We use the CIBP prior with the nonlinear
    Gaussian belief network so each unit can additionally vary its behavior
    between discrete and continuous representations.  We provide Markov chain
    Monte Carlo algorithms for inference in these belief networks and explore
    the structures learned on several image data sets.
  \end{abstract}
\end{frontmatter}

\section{Introduction}
\label{sec:introduction}
The belief network or directed probabilistic graphical model
\citep{pearl-1988a} is a popular and useful way to represent complex
probability distributions.  Methods for learning the parameters of such
networks are well-established.  Learning network structure, however, is
more difficult, particularly when the network includes unobserved hidden
units. Then, not only must the structure (edges) be determined, but the
number of hidden units must also be inferred. This paper contributes a
novel nonparametric Bayesian perspective on the general problem of learning
graphical models with hidden variables. Nonparametric Bayesian approaches
to this problem are appealing because they can avoid the difficult
computations required for selecting the appropriate \emph{a posteriori}
dimensionality of the model. Instead, they introduce an infinite number of
parameters into the model \emph{a priori} and inference determines the
subset of these that actually contributed to the observations. The Indian
buffet process
(IBP)~\citep{griffiths-ghahramani-2006a,ghahramani-etal-2007a} is one
example of a nonparametric Bayesian prior and it has previously been used
to introduce an infinite number of hidden units into a belief network with
a single hidden layer~\citep{wood-etal-2006a}.

This paper unites two important areas of research: nonparametric Baye-sian
methods and deep belief networks. To date, work on deep belief networks has
not addressed the general structure-learning problem. We therefore present
a unifying framework for solving this problem using nonparametric Bayesian
methods. We first propose a novel extension to the Indian buffet process
--- the cascading Indian buffet process (CIBP) --- and use the
Foster-Lyapunov criterion to prove convergence properties that make it
tractable with finite computation. We then use the CIBP to generalize the
single-layered, IBP-based, directed belief network to construct
multi-layered networks that are both infinitely wide and infinitely deep,
and discuss useful properties of such networks including expected in-degree
and out-degree for individual units. Finally, we combine this framework
with the powerful continuous sigmoidal belief network
framework~\citep{frey-1997a}. This allows us to infer the type (i.e.,
discrete or continuous) of individual hidden units---an important property
that is not widely discussed in previous work. To summarize, we present a
flexible, nonparametric framework for directed deep belief networks that
permits inference of the number of hidden units, the directed edge
structure between units, the depth of the network and the most appropriate
type for each unit.

\section{Finite Belief Networks}
\label{sec:finite-bns}
We consider belief networks that are layered directed acyclic graphs with
both visible and hidden units.  Hidden units are random variables that
appear in the joint distribution described by the belief network but are
not observed.  We index layers by~$m$, increasing with depth up to~$M$, and
allow visible units (i.e.,\ observed variables) only in layer~${m\!=\!0}$.
We require that units in layer~$m$ have parents only in layer~${m\!+\!1}$.
Within layer~$m$, we denote the number of units as~$K^{(m)}$ and index the
units with~$k$ so that the~$k$th unit in layer~$m$ is denoted~$u^{(m)}_k$.
We use the notation~$\bu^{(m)}$ to refer to the vector of all~$K^{(m)}$
units for layer~$m$ together.  A binary~${K^{(m-1)}\!\times\!K^{(m)}}$
matrix~$\bZ^{(m)}$ specifies the edges from layer~$m$ to layer~${m\!-\!1}$,
so that element~${Z^{(m)}_{k,k'}\!=\!1}$ iff there is an edge from
unit~$u^{(m)}_{k'}$ to unit~$u^{(m-1)}_{k}$.

A unit's activation is determined by a weighted sum of its parent
units.  The weights for layer~$m$ are denoted by
a~${K^{(m-1)}\!\times\!K^{(m)}}$ real-valued matrix~$\bW^{(m)}$, so
that the activations for the units in layer~$m$ can be
written as~${\by^{(m)}\!=\!(\bW^{(m+1)}\!\odot\!\bZ^{(m+1)})
  \bu^{(m+1)}\!+\!  \bgamma^{(m)}}$, where~$\bgamma^{(m)}$ is
a~$K^{(m)}$-dimen-sional vector of \emph{bias weights} and the binary
operator~$\odot$ indicates the Hada-mard (elementwise) product.

To achieve a wide range of possible behaviors for the units, we use the \textit{nonlinear Gaussian belief network}
(NLGBN) \citep{frey-1997a,frey-hinton-1999a} framework.  In the NLGBN,
the distribution on~$u^{(m)}_k$ arises from adding zero
mean Gaussian noise with precision~$\nu^{(m)}_k$ to the activation
sum~$y^{(m)}_k$.  This noisy sum is then transformed with a sigmoid
function~$\sigma(\cdot)$ to arrive at the value of the unit.  We
modify the NLGBN slightly so that the sigmoid function is from the
real line to~$(-1,1)$, i.e.~\({\sigma:\reals\!\to\!(-1,1)}\),
via ${\sigma(x)=2/(1+\exp\{x\})-1}$.  The distribution of~$u^{(m)}_k$
given its parents is then
\begin{align*}
  p(u^{(m)}_k| y^{(m)}_k, \nu^{(m)}_k) =
  \frac{
    \exp\left\{\!-\frac{\nu^{(m)}_k}{2}\!
      \left[\sigma^{-1}(u^{(m)}_k)\!-\!y^{(m)}_k\right]^2\right\}}
  {\sigma'(\sigma^{-1}(u^{(m)}_k))\sqrt{2\pi/\nu^{(m)}_k}}
\end{align*}
where~\(\sigma'(x)\!=\!\frac{\mathrm{d}}{\mathrm{d}x}\sigma(x)\).  As discussed
in \cite{frey-1997a} and shown in Figure~\ref{fig:modes},
different choices of~$\nu^{(m)}_k$ yield different belief unit behaviors
from effectively discrete binary units to nonlinear continuous units.
In
the multilayered construction we have described here, the joint
distribution over the units in a NLGBN is
\begin{multline}
  p(\{\bu^{(m)}\}^M_{m=0}\given\{\bZ^{(m)},\bW^{(m)}\}^M_{m=1},
  \{\bgamma^{(m)}, \{\nu^{(m)}_k\}^{K^{(m)}}_{k=1}\}^M_{m=0}, ) =\\
  \left[\prod^{K^{(M)}}_{k=1}p(u^{(M)}\given\gamma^{(M)}_k,\nu^{(M)}_k)\right]
  \prod^{M-1}_{m=0}\prod^{K^{(m)}}_{k=1}
  p(u^{(m)}_k\given y^{(m)}_k, \nu^{(m)}_k).
\end{multline}

\section{Infinite Belief Networks}
\label{sec:model}
Conditioned on the number of layers~$M$, the layer widths~$K^{(m)}$ and the
network structures~$\bZ^{(m)}$, inference in belief networks can be
straightforwardly implemented using Markov chain Monte
Carlo~\citep{neal-1992a}.  Learning the depth, width and structure,
however, presents significant computational challenges.  In this section,
we present a novel nonparametric prior, the \emph{cascading Indian buffet
  process}, for multi-layered belief networks that are both infinitely wide
and infinitely deep.  By using an infinite prior we avoid the need for the
complex dimensionality-altering proposals that would otherwise be required
during inference.

\begin{figure}
  \centering%
  \subfloat[$\nu=\half$]{%
    \includegraphics[width=0.3\textwidth]{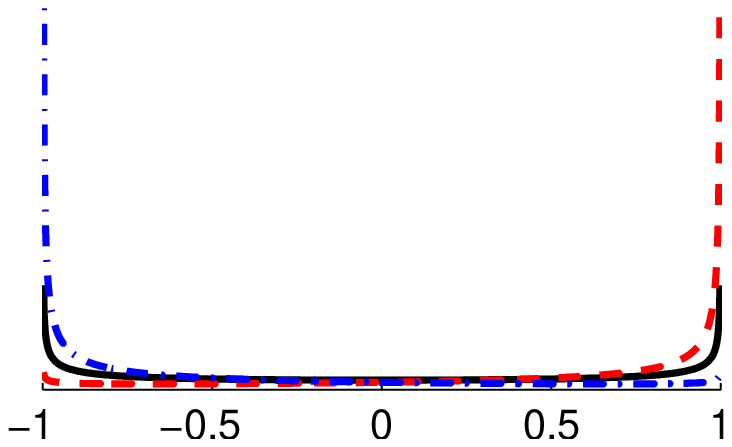}%
  }~%
  \subfloat[$\nu=5$]{%
    \includegraphics[width=0.3\textwidth]{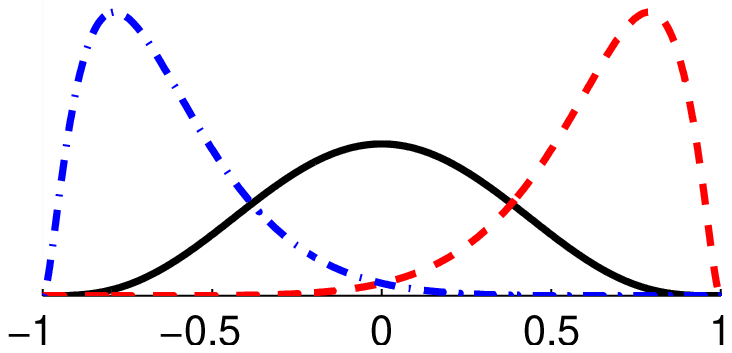}%
  }~%
  \subfloat[$\nu=1000$]{%
    \includegraphics[width=0.3\textwidth]{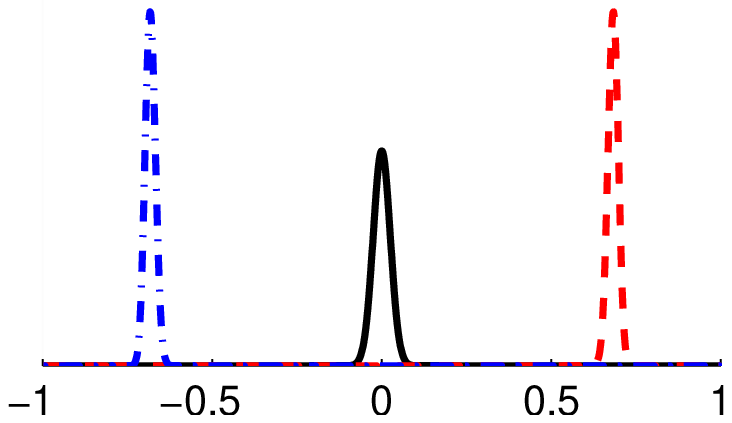}%
  }
    \caption{\small Three modes of operation for the NLGBN unit. The black
      solid line shows the zero mean distribution (i.e.\ $y\!=\!0$), the
      red dashed line shows a pre-sigmoid mean of $+1$ and the blue
      dash-dot line shows a pre-sigmoid mean of~$-1$.  (a)~Binary behavior
      from small precision.  (b)~Roughly Gaussian behavior from medium
      precision.  (c)~Deterministic behavior from large precision.}%
  \label{fig:modes}%
\end{figure}

\subsection{The Indian buffet process}
\label{sec:ibp}
Section~\ref{sec:finite-bns} used the binary matrix~$\bZ^{(m)}$ as a
convenient way to represent the edges connecting layer~$m$ to
layer~${m\!-\!1}$.  We stated that~$\bZ^{(m)}$ was a
finite~${K^{(m-1)}\!\times\!K^{(m)}}$ matrix.  We can use the
\textit{Indian buffet process} (IBP)~\citep{griffiths-ghahramani-2006a} to
allow this matrix to have an infinite number of columns.  We assume the
two-parameter IBP \citep{ghahramani-etal-2007a}, and
use~${\bZ^{(m)}\!\sim\!  \distIBP(\alpha,\beta)}$ to indicate that the
matrix~\(\bZ^{(m)}\!\in\!\{0,1\}^{K^{(m-1)}\times\infty}\) is drawn from an
IBP with parameters~$\alpha,\beta>0$.  The eponymous metaphor for the IBP
is a restaurant with an infinite number of dishes available.  Each customer
chooses a finite set of dishes to taste.  The rows of the binary matrix
correspond to customers and the columns correspond to dishes.  If the~$j$th
customer tastes the~$k$th dish, then~${Z_{j,k}\!=\!1}$,
otherwise~${Z_{j,k}\!=\!0}$.  The first customer into the restaurant
samples a number of dishes that is Poisson distributed with
parameter~$\alpha$.  After that, when the~$j$th customer enters the
restaurant, she selects dish~$k$ with
probability~$\eta_k/(j\!+\!\beta\!-\!1)$, where~$\eta_k$ is the number of
previous customers that have tried the~$k$th dish.  She then chooses a
number of additional dishes to taste that is Poisson distributed with
parameter~$\alpha\beta/(j\!+\!\beta\!-\!1)$.  Even though each customer
chooses dishes based on their popularity with previous customers, the rows
and columns of the resulting matrix~$\bZ^{(m)}$ are infinitely
exchangeable.

As in \cite{wood-etal-2006a}, if the model of Section~\ref{sec:finite-bns}
had only a single hidden layer, i.e.\ \(M\!=\!1\), then the IBP could be
used to make that layer infinitely wide.  While a belief network with an
infinitely-wide hidden layer can represent any probability distribution
arbitrarily closely~\citep{leroux-bengio-2008a}, it is not necessarily a
useful prior on such distributions.  Without intra-layer connections, the
the hidden units are independent \textit{a priori}.  This ``shallowness''
is a strong assumption that weakens the model in practice and the explosion
of recent literature on \textit{deep belief networks} (see, e.g.
\cite{hinton-etal-2006a,hinton-salakhutdinov-2006a})
speaks to the empirical success of belief networks with more hidden
structure.

\subsection{The cascading Indian buffet process}
\label{sec:cibp}
To build a prior on belief networks that are unbounded in both width and
depth, we use an IBP-like object that provides an infinite sequence of
binary matrices~\({\bZ^{(0)}, \bZ^{(1)}, \bZ^{(2)}, \cdots}\).  We require
the matrices in this sequence to inherit the useful sparsity properties of
the IBP, with the constraint that the columns from~$\bZ^{(m-1)}$ correspond
to the rows in~$\bZ^{(m)}$.  We interpret each matrix~$\bZ^{(m)}$ as
specifying the directed edge structure from layer~$m$ to layer~${m\!-\!1}$,
where both layers have a potentially-unbounded width.

We propose the cascading Indian buffet process to provide a prior with
these properties.  The CIBP extends the vanilla IBP in the following way:
each of the ``dishes'' in the restaurant are also ``customers'' in another
Indian buffet process.  The columns in one binary matrix correspond to the
rows in another binary matrix.  The CIBP is infinitely exchangeable in the
rows of matrix~$\bZ^{(0)}$.  Each of the IBPs in the recursion is
exchangeable in its rows and columns, so it does not change the probability
of the data to propagate a permutation back through the matrices.

If there are~$K^{(0)}$ customers in the first restaurant, a surprising
result is that, for finite~$K^{(0)}$,~$\alpha$, and~$\beta$, the CIBP
recursion terminates with probability one.  By ``terminate'' we mean that
at some point the customers do not taste any dishes and all deeper
restaurants have neither dishes nor customers.  Here we only sketch the
intuition behind this result.  A proof is provided in Appendix~\ref{sec:appendix}.

The matrices in the CIBP are constructed in a sequence, starting
with ${m\!=\!0}$.  The number of nonzero columns in matrix~$\bZ^{(m+1)}$,
$K^{(m+1)}$, is determined entirely by~$K^{(m)}$, the number of active
nonzero columns in~$\bZ^{(m)}$.  We require that for some
matrix~$\bZ^{(m)}$, there are no nonzero columns.  For this purpose, we can
disregard the fact that it is a matrix-valued stochastic process and
instead consider the Markov chain that results on the number of nonzero
columns.  Figure~\ref{fig:depth-traces} shows three traces of such a Markov
chain on~$K^{(m)}$.  If we
define~${\lambda(K;\alpha,\beta)=\alpha\sum^{K}_{k'=1}\frac{\beta}{k'+\beta-1}}$,
then the Markov chain has the transition distribution
\begin{align}
  \label{eqn:markov-transitions}
  p(K^{(m+1)}=k\given K^{(m)},\alpha,\beta)
  &=
  \frac{1}{k!}
  \exp\left\{-\lambda(K^{(m)};\alpha,\beta)\right\}
  \lambda(K^{(m)};\alpha,\beta)^k,
\end{align}
which is simply a Poisson distribution with
mean~$\lambda(K^{(m)};\alpha,\beta)$.  Clearly, ${K^{(m)}=0}$ is an
absorbing state, however, the state space of the Markov chain is
countably-infinite and to know that it will reach the absorbing state with
probability one, we must know that~$K^{(m)}$ does not blow up to infinity.

\begin{figure}[t!]
  \centering%
  \subfloat[{\small Example traces with~${K^{(0)}=50}$}]{%
    \centering%
    \includegraphics[width=0.95\textwidth]{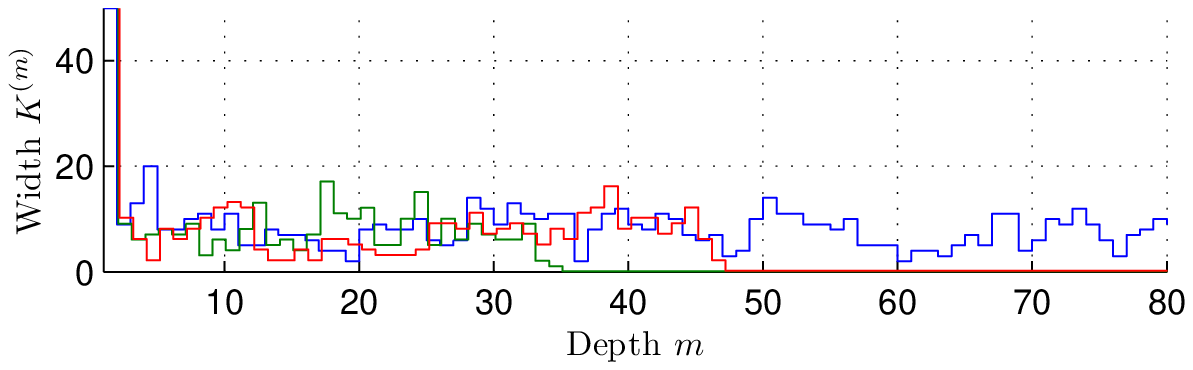}%
    \label{fig:depth-traces}%
  }\\
  \subfloat[{\small Expected~$K^{(m+1)}$}]{%
    \centering%
    \includegraphics[width=0.45\textwidth]{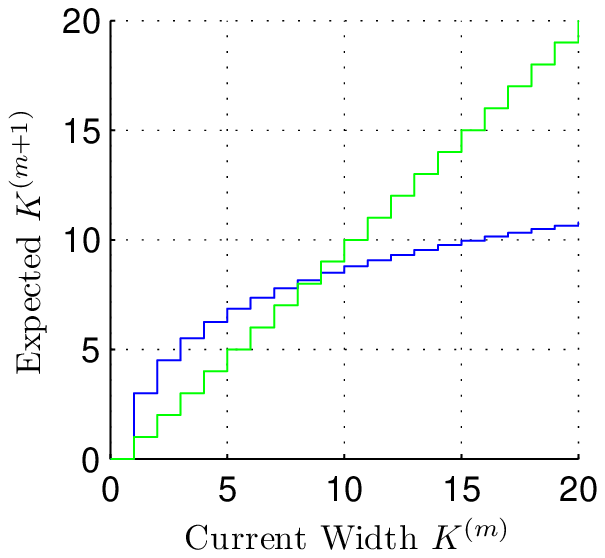}%
    \label{fig:flsc}%
  }~%
  \subfloat[{\small Drift}]{%
    \centering%
    \includegraphics[width=0.45\textwidth]{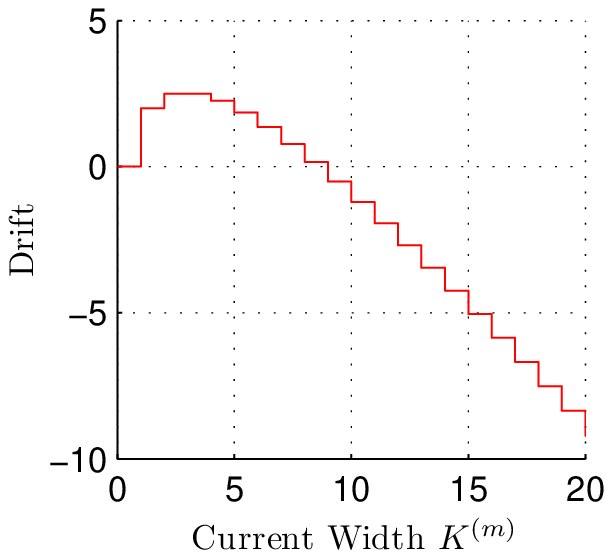}%
    \label{fig:drift}%
  }
  \caption{\small Properties of the Markov chain on layer width for the
    CIBP, with~${\alpha=3}$, ${\beta=1}$.  Note that these values are
    illustrative and are not necessarily appropriate for a network
    structure.  a)~Example traces of a Markov chain on layer width, indexed
    by depth~$m$.  b)~Expected~$K^{(m+1)}$ as a function of~$K^{(m)}$ is
    shown in blue.  The Lyapunov function~$\mcL(\cdot)$ is shown in green.
    c)~The drift as a function of the current width~$K^{(m)}$.  This
    corresponds to the difference between the two lines in (a).  Note that
    it goes negative when the layer width is greater than eight.}
\end{figure}

In such a Markov chain, this requirement is equivalent to the statement
that the chain has an equilibrium distribution when conditioned on
nonabsorption (has a \textit{quasi-stationary distribution})
\citep{seneta-verejones-1966a}.  For countably-infinite state spaces, a
Markov chain has a (quasi-) stationary distribution if it is
positive-recurrent, which is the property that there is a finite expected
time between consecutive visits to any state.  Positive recurrency can be
shown by proving the \textit{Foster--Lyapunov stability criterion}
(FLSC)~\citep{fayolle-etal-2008a}.  Taken together, satisfying the FLSC for
the Markov chain with transition probabilities given by
Eqn~\ref{eqn:markov-transitions} demonstrates that eventually the CIBP will
reach a restaurant in which the customers try no new dishes.  We do this by
showing that if~$K^{(m)}$ is large enough, the expected~$K^{(m+1)}$ is
smaller than~$K^{(m)}$.

The FLSC requires a \textit{Lyapunov
  function}~${\mcL(k):\naturals^{+}\to\reals\geq 0}$, with which we define the
\emph{drift function}:
\begin{align*}
  \expectation_{k|K^{(m)}}[ \mcL(k) - \mcL(K^{(m)})] &=
  \sum^{\infty}_{k=1} p(K^{(m+1)}=k\given K^{(m)})(\mcL(k)-\mcL(K^{(m)})).
\end{align*}
The drift is the expected change in~$\mcL(k)$.  If there is a~$K^{(m)}$
above which all drifts are negative, then the Markov chain satisfies the
FLSC and is positive-recurrent.  In the CIBP, this is satisfied
for~${\mcL(k)=k}$.  That the drift eventually becomes negative can be seen
by the fact
that
\begin{align*}
  \expectation_{k|K^{(m)}}[\mcL(k)] &= \lambda(K^{(m)}\,;\,\alpha,\beta)
\end{align*}
is~$O(\ln K^{(m)})$ and~${\expectation_{k|K^{(m)}}[\mcL(K^{(m)})]=K^{(m)}}$
is~$O(K^{(m)})$.  Figures~\ref{fig:flsc} and~\ref{fig:drift} show a
schematic of this idea.

\begin{figure}[!t]
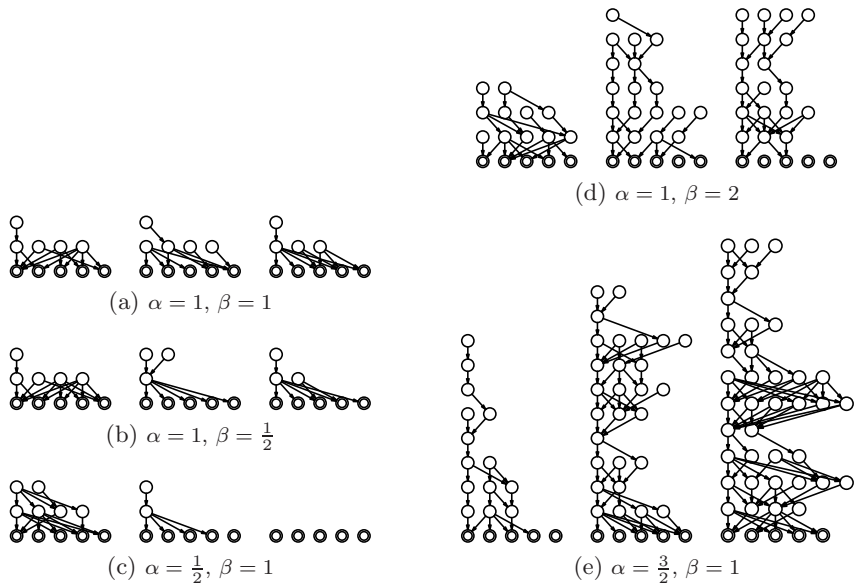

  \centering%
  \begin{minipage}[b]{0.49\textwidth}%
    \centering%
    \subfloat[\scriptsize{$\alpha=1$, $\beta=1$}]{%
      \includegraphics[width=1.4cm]{figures/samples/fig_a1.1}\quad%
      \includegraphics[width=1.4cm]{figures/samples/fig_a2.1}\quad%
      \includegraphics[width=1.4cm]{figures/samples/fig_a3.1}%
    }\\
    \subfloat[\scriptsize{$\alpha=1$, $\beta=\half$}]{%
      \includegraphics[width=1.4cm]{figures/samples/fig_d1.1}\quad%
      \includegraphics[width=1.4cm]{figures/samples/fig_d2.1}\quad%
      \includegraphics[width=1.4cm]{figures/samples/fig_d3.1}%
    }\\%
    \subfloat[\scriptsize{$\alpha=\half$, $\beta=1$}]{%
      \includegraphics[width=1.4cm]{figures/samples/fig_e1.1}\quad%
      \includegraphics[width=1.4cm]{figures/samples/fig_e2.1}\quad%
      \includegraphics[width=1.4cm]{figures/samples/fig_e3.1}%
    }%
  \end{minipage}%
  \begin{minipage}[b]{0.49\textwidth}%
    \centering%
    \subfloat[\scriptsize{$\alpha=1$, $\beta=2$}]{%
      \includegraphics[width=1.4cm]{figures/samples/fig_b1.1}\quad%
      \includegraphics[width=1.4cm]{figures/samples/fig_b2.1}\quad%
      \includegraphics[width=1.4cm]{figures/samples/fig_b3.1}%
    }\\
    \subfloat[\scriptsize{$\alpha=\frac{3}{2}$, $\beta=1$}]{%
      \includegraphics[width=1.4cm]{figures/samples/fig_c1.1}\quad%
      \includegraphics[width=1.4cm]{figures/samples/fig_c2.1}\quad%
      \includegraphics[width=1.8cm]{figures/samples/fig_c3.1}%
    }%
  \end{minipage}
  \caption{\small Samples from the CIBP-based prior on network structures,
    with five visible units.}%
  \label{fig:samples}%
\end{figure}

\subsection{The CIBP as a prior on the structure of an infinite belief
  network} The CIBP can be used as a prior on the sequence~\({\bZ^{(0)},
  \bZ^{(1)}, \bZ^{(2)}, \cdots}\) from Section~\ref{sec:finite-bns}, to
allow an infinite sequence of infinitely-wide hidden layers.  As before,
there are~$K^{(0)}$ visible units.  The edges between the first hidden
layer and the visible layer are drawn according to the restaurant metaphor.
This yields a finite number of units in the first hidden layer,
denoted~$K^{(1)}$ as before.  These units are now treated as the visible
units in another IBP-based network.  While this recurses infinitely deep,
only a finite number of units are ancestors of the visible units.
Figure~\ref{fig:samples} shows several samples from the prior for different
parameterizations.  Only connected units are shown in the figure.

The parameters~$\alpha$ and~$\beta$ govern the expected width and sparsity
of the network at each level. The expected in-degree of each unit (number
of parents) is~$\alpha$ and the expected out-degree (number of children)
is ${K/\sum^K_{k=1}\frac{\beta}{\beta+k-1}}$, for~$K$ units used in the
layer below.  For clarity, we have presented the CIBP results with~$\alpha$
and~$\beta$ fixed at all depths; however, this may be overly restrictive.
For example, in an image recognition problem we would not expect the
sparsity of edges mapping low-level features to pixels to be the same as
that for high-level features to low-level features.  To address this, we
allow~$\alpha$ and~$\beta$ to vary with depth, writing~$\alpha^{(m)}$
and~$\beta^{(m)}$.  The CIBP terminates with probability one as long as
there exists some finite upper bound for~$\alpha^{(m)}$ and~$\beta^{(m)}$
for all~$m$.

\subsection{Priors on other parameters}
Other parameters in the model also require prior distributions and we use
these priors to tie parameters together according to layer.  We assume that
the weights in layer~$m$ are drawn independently from Gaussian
distributions with mean~$\mu^{(m)}_w$ and precision~$\rho^{(m)}_m$.  We
assume a similar layer-wise prior for biases with
parameters~$\mu^{(m)}_\gamma$ and~$\rho^{(m)}_\gamma$.  We use layer-wise
gamma priors on the~$\nu^{(m)}_k$, with parameters~$a^{(m)}$ and~$b^{(m)}$.
We tie these prior parameters together with global normal-gamma hyperpriors
for the weight and bias parameters, and gamma hyperpriors for the precision
parameters.

\section{Inference}
We have so far described a prior on belief network structure and
parameters, along with likelihood functions for unit activation.  The
inference task in this model is to find the posterior distribution
over the structure and the parameters of the network, having
seen~$N$~$K^{(0)}$-dimensional vectors~${\{\bx_n \in
  (-1,1)^{K^{(0)}}\}^N_{n=1}}$.  This posterior distribution is
complex, so we use Markov chain Monte Carlo (MCMC) to draw samples
from~${p( \{ \bZ^{(m)}, \bW^{(m)}\}^{\infty}_{m}, \{\bgamma^{(m)},
  \bnu^{(m)}\}^{\infty}_{m}, \{\bx_n\}^N_{n} )}$, which, for
fixed~\(\{\bx_n\}^N_{n}\), is proportional to the posterior
distribution.  This joint distribution requires marginalizing over the
states of the hidden units that led to each of the~$N$ observations.
The values of these hidden units are
denoted~$\{\{\bu^{(m)}_n\}^{\infty}_{m=1}\}^N_{n=1}$, and we augment
the Markov chain to include these as well.

In general, one would not expect that a distribution on infinite networks
would yield tractable inference.  However, in our construction, conditioned
on the sequence~\({\bZ^{(1)}, \bZ^{(2)},\cdots}\), almost all of the
infinite number of units are independent.  Due to this independence, they
trivially marginalize out of the model's joint distribution and we can
restrict inference only to those units that are ancestors of the visible
units.  Of course, since this trivial marginalization only arises from
the~$\bZ^{(m)}$ matrices, we must also have a distribution on infinite
binary matrices that allows exact marginalization of all the uninstantiated
edges.  The row-wise and column-wise exchangeability properties of the IBP
are what allows the use of infinite matrices.  The bottom-up conditional
structure of the CIBP allows an infinite number of these matrices.

To simplify notation, we will use~$\bOmega$ for the aggregated state of the
model variables, i.e.~${\bOmega\!=\!(\{ \bZ^{(m)}, \bW^{(m)},
  \{\bu^{(m)}_n\}^N_{n=1}\}^{\infty}_{m=1}, \{\bgamma^{(m)},
  \bnu^{(m)}\}^{\infty}_{m=0}, \{\bx_n\}^N_{n=1})}$.  Given the
  hyperparameters, we can then write the joint distribution
as
\begin{multline}
  p(\bOmega) =
  \left(
  p(\bgamma^{(0)})\,
  p(\bnu^{(0)})
  \prod^{K^{(0)}}_{k=1}
  \prod^{N}_{n=1}
  p(x_{k,n}\given y^{(0)}_{k,n}, \nu^{(0)}_k)\right)\\
\times
\left(
  \prod^{\infty}_{m=1}
  p(\bW^{(m)})\,
  p(\bgamma^{(m)})\,
  p(\bnu^{(m)})
  \prod^{K^{(m)}}_{k=1}
  \prod^{N}_{n=1}
  p(u^{(m)}_{k,n}\given y^{(m)}_{k,n}, \nu^{(m)}_k)\right).
\end{multline}
Although this distribution involves several infinite sets, it is possible
to sample from the relevant parts of the posterior.  We do this by MCMC,
updating part of the model, while conditioning on the rest.  In particular,
conditioning on the binary matrices~${\{\bZ^{(m)}\}^{\infty}_{m=1}}$, which
define the structure of the network, inference becomes exactly as it would
be in a finite belief network.

\subsection{Sampling from the hidden unit states}
Since we cannot easily integrate out the hidden units, it is necessary to
explicitly represent them and sample from them as part of the Markov chain.
As we are conditioning on the network structure, it is only necessary to
sample the units that are ancestors of the visible units.
\citet{frey-1997a} proposed a slice sampling scheme for the hidden unit
states but we have been more successful with a specialized
independence-chain variant of multiple-try
Metropolis--Hastings~\citep{liu-etal-2000a}.  Our method proposes several
($\approx 5$) possible new unit states from the activation distribution and
selects from among them (or rejects them all) according to the likelihood
imposed by its children.  As this operation can be executed in parallel by
tools such as Matlab, we have seen significantly better mixing
performance by wall-clock time than the slice sampler.

\subsection{Sampling from the weights and biases}
Given that a directed edge exists, we sample the posterior distribution
over its weight.  Conditioning on the rest of the model, the NLGBN results
in a convenient Gaussian form for the distribution on weights so that we
can Gibbs update them using a Gaussian with parameters
\begin{align}
  \mu^{\sf{w-post}}_{m,k,k'}  &\!=\! \frac{\rho^{(m)}_w\!\mu^{(m)}_w 
    \!\!+\! \nu^{(m\!-\!1)}_k\! \sum_{n}\!\! u^{(m)}_{n,k'}
    (\sigma^{-\!1}(u^{(m\!-\!1)}_{k})\!-\!\xi^{(m)}_{n,k,k'})}{
  \rho^{(m)}_w+\nu^{(m-1)}_k\sum_{n}(u^{(m)}_{n,k'})^2} \\
  \rho^{\sf{w-post}}_{m,k,k'} &\!=\! 
  \rho^{(m)}_w+\nu^{(m-1)}_k\sum_{n}(u^{(m)}_{n,k'})^2,
\end{align}
where
\begin{align}
  \xi^{(m)}_{n,k,k'} &= \gamma^{(m-1)}_{k} 
  + \sum_{k''\neq k'}Z^{(m)}_{k,k''}W^{(m)}_{k,k''}u^{(m)}_{n,k''}.
\end{align}
The bias~$\gamma^{(m)}_k$ can be similarly sampled from a Gaussian
distribution with parameters
\begin{align}
  \mu^{\gamma\sf{-post}}_{m,k}  &=
  \frac{
    \rho^{(m)}_\gamma\mu^{(m)}_\gamma
    + \nu^{(m)}_k\sum^N_{n=1}(\sigma^{-1}(u^{(m)}_{n,k})-\chi^{(m)}_{n,k})
  }{
    \rho^{(m)}_{\gamma} + N\nu^{(m)}_k
  }\\
  \rho^{\gamma\sf{-post}}_{m,k} &= \rho^{(m)}_{\gamma} + N\nu^{(m)}_k
\end{align}
where
\begin{align}
\chi^{(m)}_{n,k} &= \sum^{K^{(m+1)}}_{k'=1}Z^{(m+1)}_{k,k'}
W^{(m+1)}_{k,k'}u^{(m+1)}_{n,k'}.
\end{align}

\subsection{Sampling from the activation variances}
We use the NLGBN model to gain the ability to change the mode of unit
behaviors between discrete and continuous representations.  This
corresponds to sampling from the posterior distributions over
the~$\nu^{(m)}_k$.  With a conjugate prior, the new value can be sampled
from a gamma distribution with parameters
\begin{align}
  a^{\nu\sf{-post}}_{m,k} &= a^{(m)}_\nu + N/2\\
  b^{\nu\sf{-post}}_{m,k} &= b^{(m)}_\nu +
  \half\sum^N_{n=1}(\sigma^{-1}(u^{(m)}_{n,k})-y^{(m)}_{k})^2.
\end{align}

\subsection{Sampling from the structure}
A model for infinite belief networks is only useful if it is possible to
perform inference.  The appeal of the CIBP prior is that it enables
construction of a tractable Markov chain for inference.  To do this
sampling, we must add and remove edges from the network, consistent with
the posterior equilibrium distribution.  When adding a layer, we must
sample additional layerwise model components.  When introducing an edge, we
must draw a weight for it from the prior.  If this new edge introduces a
previously-unseen hidden unit, we must draw a bias for it and also draw its
deeper cascading connections from the prior.  Finally, we must sample
the~$N$ new hidden unit states from any new unit we introduce. 

We iterate over each layer that connects to the visible units.  Within each
layer~${m \geq 0}$, we iterate over the connected units.  Sampling the
edges incident to the~$k$th unit in layer~$m$ has two phases.  First, we
iterate over each connected unit in layer~${m+1}$, indexed by~$k'$.  We
calculate~$\eta^{(m)}_{-k,k'}$, which is the number of nonzero entries in
the~$k'$th column of~$\bZ^{(m+1)}$, excluding any entry in the~$k$th row.
If~$\eta^{(m)}_{-k,k'}$ is zero, we call the unit~$k'$ a \textit{singleton}
parent, to be dealt with in the second phase.  If~$\eta^{(m)}_{-k,k'}$ is
nonzero, we introduce (or keep) the edge from unit~$u^{(m+1)}_{k'}$
to~$u^{(m)}_k$ with Bernoulli probability
\begin{multline*}
  p(Z^{(m+1)}_{k,k'} = 1\given\bOmega\backslash Z^{(m+1)}_{k,k'})
  = \frac{1}{\mcZ}
  \left(\!\frac{\eta^{(m)}_{-k,k'}}{K^{(m)} \!+\! \beta^{(m)} \!-\! 1}\!\right)\\
  \times \prod^N_{n=1}
  p(u^{(m)}_{n,k}\given Z^{(m+1)}_{k,k'}=1, \bOmega\backslash
  Z^{(m)}_{k,k'})
\end{multline*}
\begin{multline*}
  p(Z^{(m+1)}_{k,k'} =0\given\bOmega\backslash Z^{(m+1)}_{k,k'})
  = \frac{1}{\mcZ}
  \left(\!1\!-\!\frac{\eta^{(m)}_{-k,k'}}{K^{(m)} \!+\! \beta^{(m)} \!-\! 1}\!\right)\\
  \times \prod^N_{n=1}
  p(u^{(m)}_{n,k}\given Z^{(m+1)}_{k,k'}=0, \bOmega\backslash Z^{(m+1)}_{k,k'}),
\end{multline*}
where~$\mcZ$ is the appropriate normalization constant.

In the second phase, we consider deleting connections to singleton parents
of unit~$k$, or adding new singleton parents.  We do this via a
Metropolis--Hastings operator using a birth/death process.  If there are
currently~$K_{\circ}$ singleton parents, then with probability~$1/2$ we
propose adding a new one by drawing it recursively from deeper layers, as
above.  We accept the proposal to insert a connection to this new parent
unit with M--H acceptance ratio
\begin{align*}
  a_{\sf{mh-insert}} &= \frac{\alpha^{(m)}\beta^{(m)}}{ (K_\circ\!+\!1)^2(\beta^{(m)}\! +\! K^{(m)}\! -\!1)}
  \prod^N_{n=1}\frac{
    p(u^{(m)}_{n,k}\given Z^{(m+1)}_{k,j}\!\!=\!1,\bOmega\backslash Z^{(m+1)}_{k,j})
    }{
    p(u^{(m)}_{n,k}\given Z^{(m+1)}_{k,j}\!\!=\!0,\bOmega\backslash Z^{(m+1)}_{k,j})
    }.
\end{align*}
If we do not propose to insert a unit and~${K_{\circ}\geq 0}$, then with
probability~$1/2$ we select uniformly from among the singleton parents of
unit~$k$ and propose removing the connection to it.  We accept the proposal
to remove the~$j$th one with M--H acceptance ratio
\begin{align*}
  a_{\sf{mh-remove}} &= \frac{K_\circ^2(\beta^{(m)}\! +\! K^{(m)}\! -\!1)}{\alpha^{(m)}\beta^{(m)} }
  \prod^N_{n=1}\frac{
    p(u^{(m)}_{n,k}\given Z^{(m+1)}_{k,j}\!\!=\!0,\bOmega\backslash Z^{(m+1)}_{k,j})
    }{
    p(u^{(m)}_{n,k}\given Z^{(m+1)}_{k,j}\!\!=\!1,\bOmega\backslash Z^{(m+1)}_{k,j})
    }.
\end{align*}
After these phases, chains of units that are not ancestors of the visible
units can be discarded.  Notably, this birth/death operator samples from
the IBP posterior with a non-truncated equilibrium distribution, even
without conjugacy.  Unlike the stick-breaking approach of
\citet{teh-etal-2007a}, it allows use of the two-parameter IBP, which is
important to this model.

\subsection{Sampling From CIBP Hyperparameters}
When applying this model to data, it is infrequently the case that we
would have a good \textit{a priori} idea of what the appropriate IBP
parameters should be.  These control the width and sparsity of the
network and while we might have good initial guesses for the lowest
layer, in general we would like to
infer~$\{\alpha^{(m)},\beta^{(m)}\}$ as part of the larger inference
procedure.  This is straightforward in the fully-Bayesian MCMC
procedure we have constructed, and it does not differ markedly from
hyperparameter inference in standard IBP models when conditioning
on~$\bZ^{(m)}$.  As in some other nonparametric models
(e.g. \citet{tokdar-2006a} and \citet{rasmussen-williams-2006a}), we
have found that light-tailed priors on the hyperparameters helps
ensure that the model stays in reasonable states.

\begin{figure}[t!]
  \centering%
  \subfloat[]{%
    \includegraphics[height=5cm]{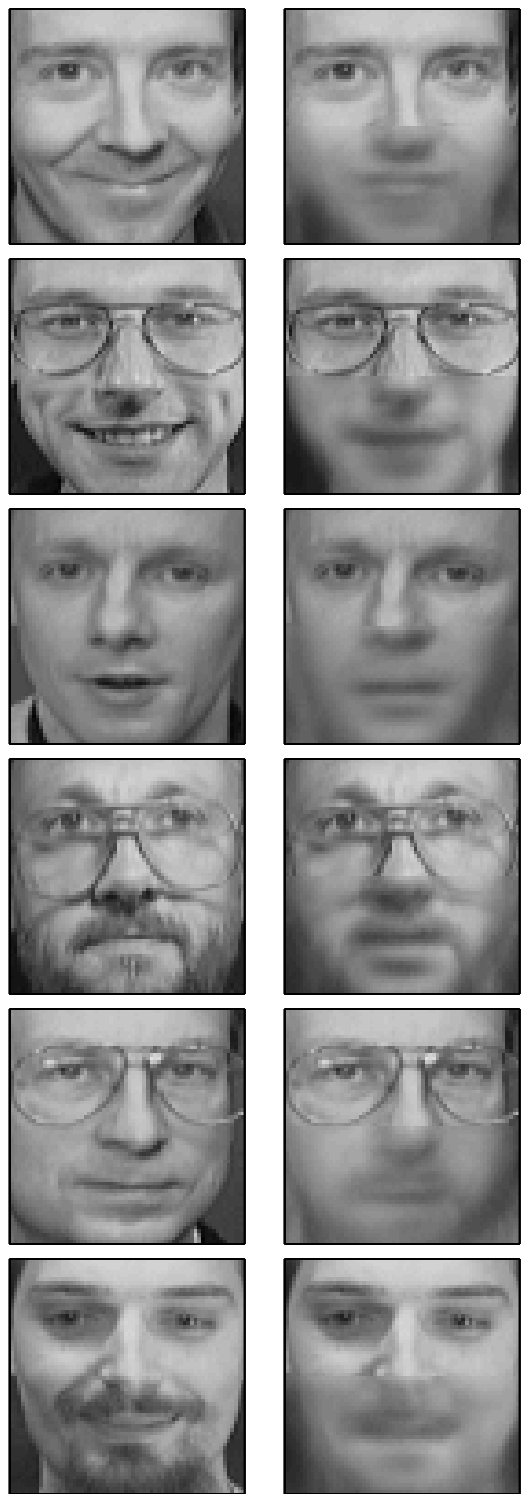}%
    \label{fig:olivetti-recons}%
  }~\quad~%
  \subfloat[]{%
    \includegraphics[height=5cm]{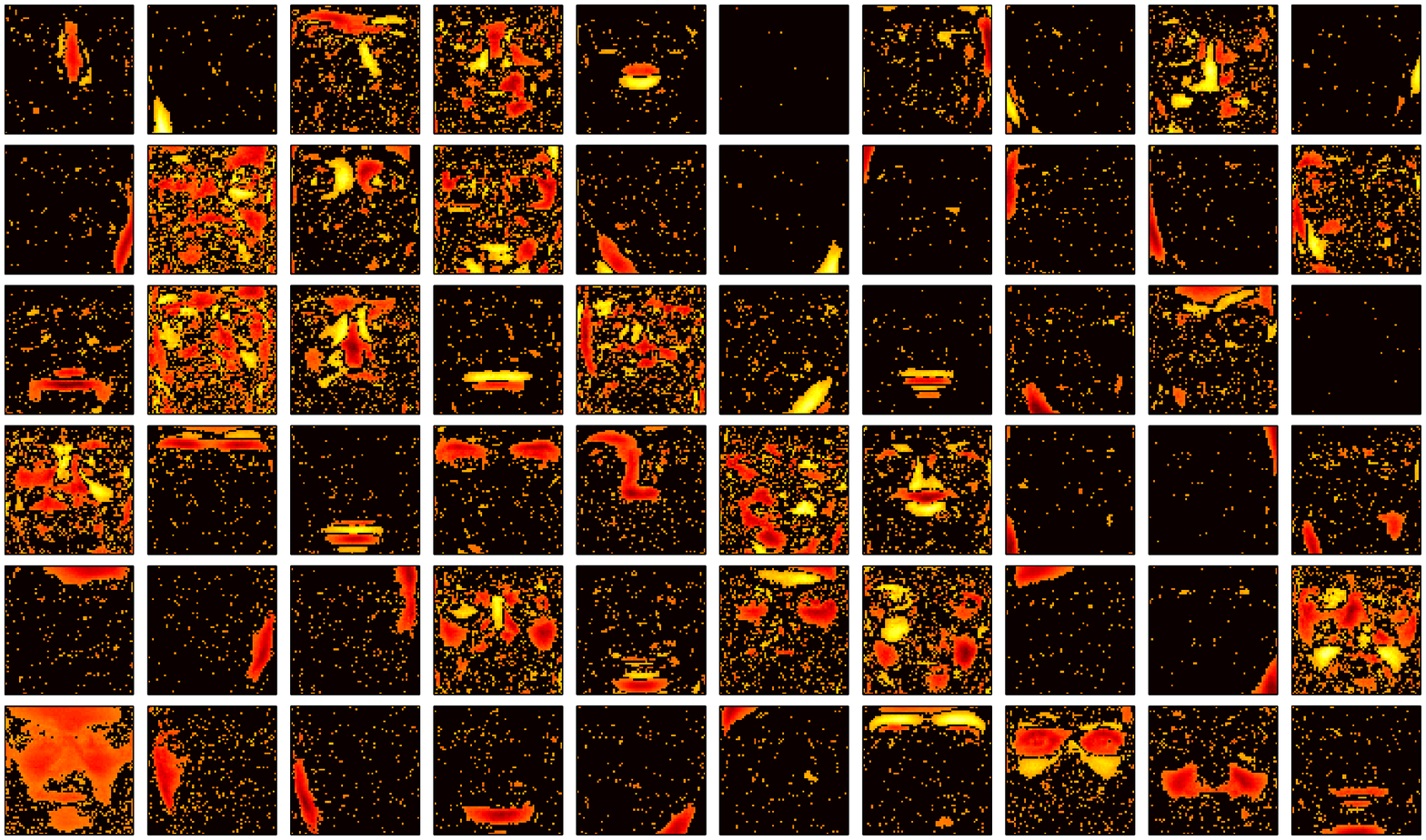}%
    \label{fig:olivetti-weights}%
  }\\%
  \subfloat[]{%
    \includegraphics[height=5cm]{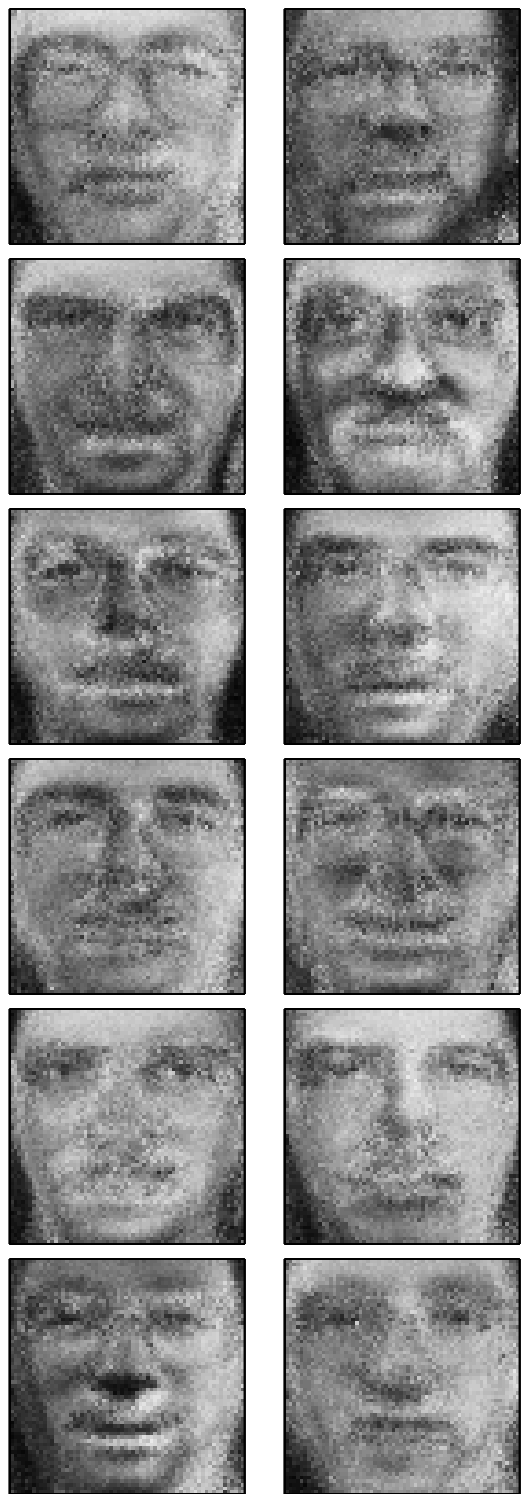}%
    \label{fig:olivetti-fantasies}%
  }~\quad~%
  \subfloat[]{%
    \includegraphics[height=5cm]{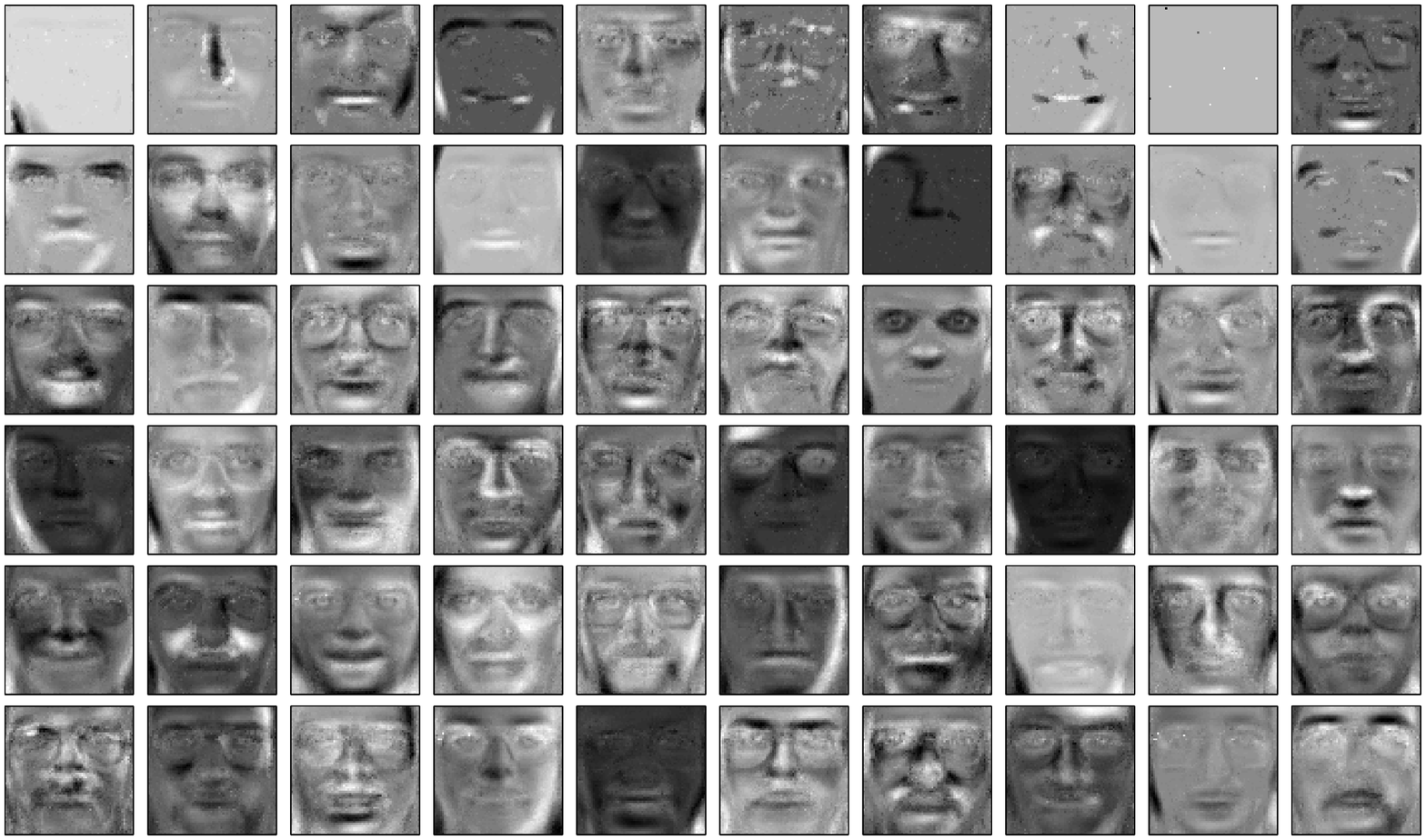}%
    \label{fig:olivetti-features}%
  }%
  \caption{\small Olivetti faces a)~Test images on the left, with
    reconstructed bottom halves on the right.  b)~Sixty features learned in
    the bottom layer, where black shows absence of an edge.  Note the
    learning of sparse features corresponding to specific facial structures
    such as mouth shapes, noses and eyebrows. c)~Raw predictive fantasies.
    d)~Feature activations from individual units in the second hidden
    layer.}
  \label{fig:olivetti}
\end{figure}

\section{Reconstructing Images}
We applied the model and MCMC-based inference procedure to three image data
sets: the Olivetti faces, the MNIST digits and the Frey faces.  We used
these data to analyze the structures and sparsity that arise in the model
posterior.  To get a sense of the utility of the model, we constructed a
missing-data problem using held-out images from each set.  We removed the
bottom halves of the test images and asked the model to reconstruct the
missing data, conditioned on the top half.  The prediction itself
was done by integrating out the parameters and structure via MCMC.

\paragraph*{Olivetti Faces}
\label{sec:olivetti}
The Olivetti faces data~\citep{samaria-harter-1994a} consists of
400~${64\times 64}$ grayscale images of the faces of 40 distinct subjects.
We divided these into 350 test data and 50 training data, selected
randomly.  This data set is an appealing test because it has few examples,
but many dimensions.  Figure~\ref{fig:olivetti-recons} shows six
bottom-half test set reconstructions on the right, compared to the ground
truth on the left.  Figure~\ref{fig:olivetti-weights} shows a subset of
sixty weight patterns from a posterior sample of the structure, with black
indicating that no edge is present from that hidden unit to the visible
unit (pixel).  The algorithm is clearly assigning hidden units to specific
and interpretable features, such as mouth shapes, the presence of glasses
or facial hair, and skin tone, while largely ignoring the rest of the
image.  Figure~\ref{fig:olivetti-fantasies} shows ten pure fantasies from
the model, easily generated in a directed acyclic belief network.
Figure~\ref{fig:olivetti-features} shows the result of activating
individual units in the second hidden layer, while keeping the rest
unactivated, and propagating the activations down to the visible pixels.
This provides an idea of the image space spanned by the principal
components of these deeper units.  A typical posterior network had three
hidden layers, with about 70 units in each layer.

\begin{figure}[t!]
  \centering%
  \subfloat[]{%
    \includegraphics[height=6cm]{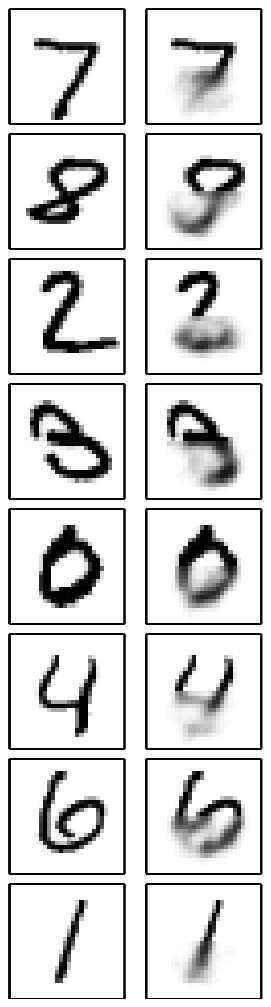}%
    \label{fig:mnist-recons}%
  }~%
  \subfloat[]{%
    \includegraphics[height=6cm]{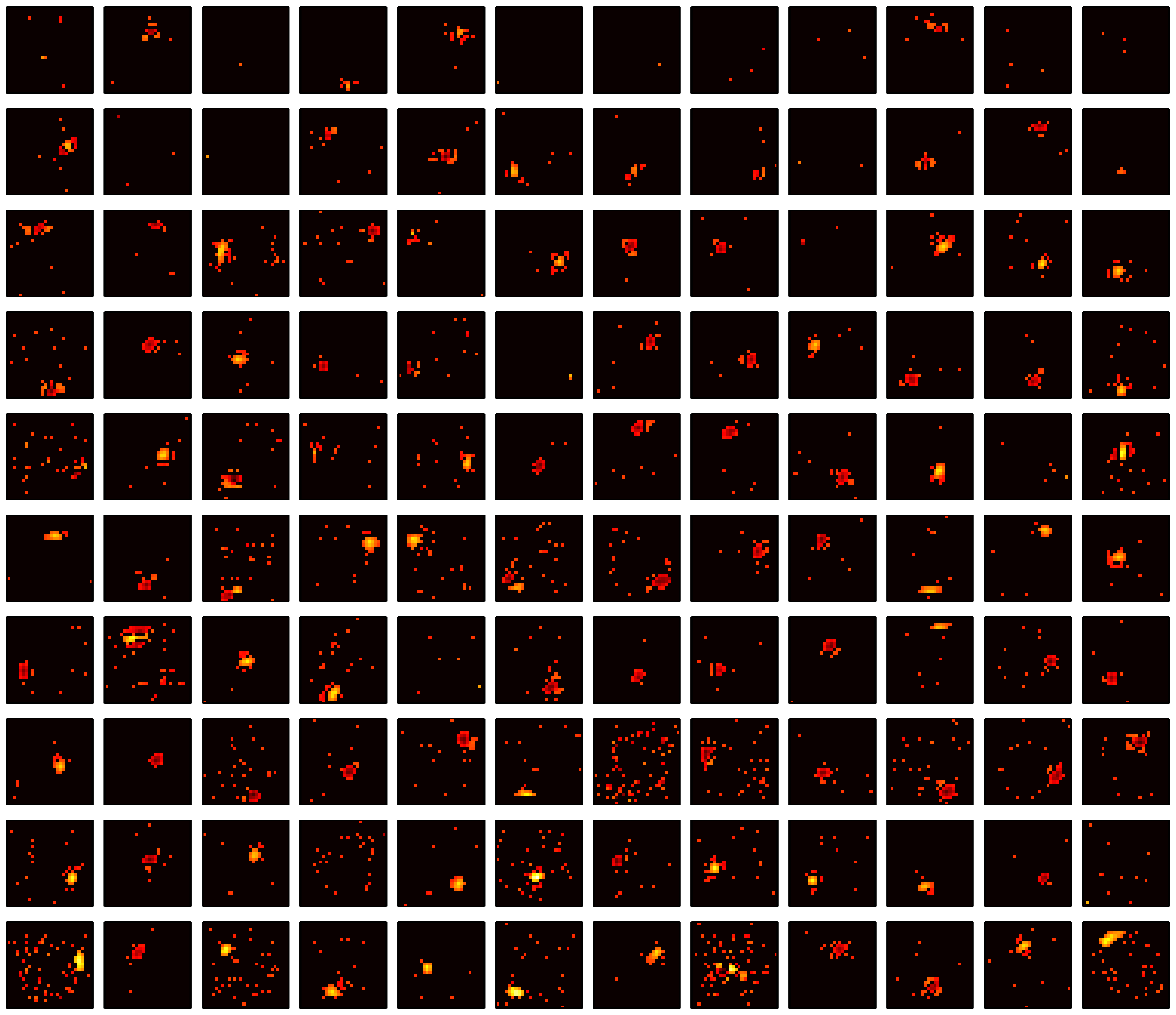}%
    \label{fig:mnist-weights}%
  }~%
  \subfloat[]{%
    \includegraphics[height=6cm]{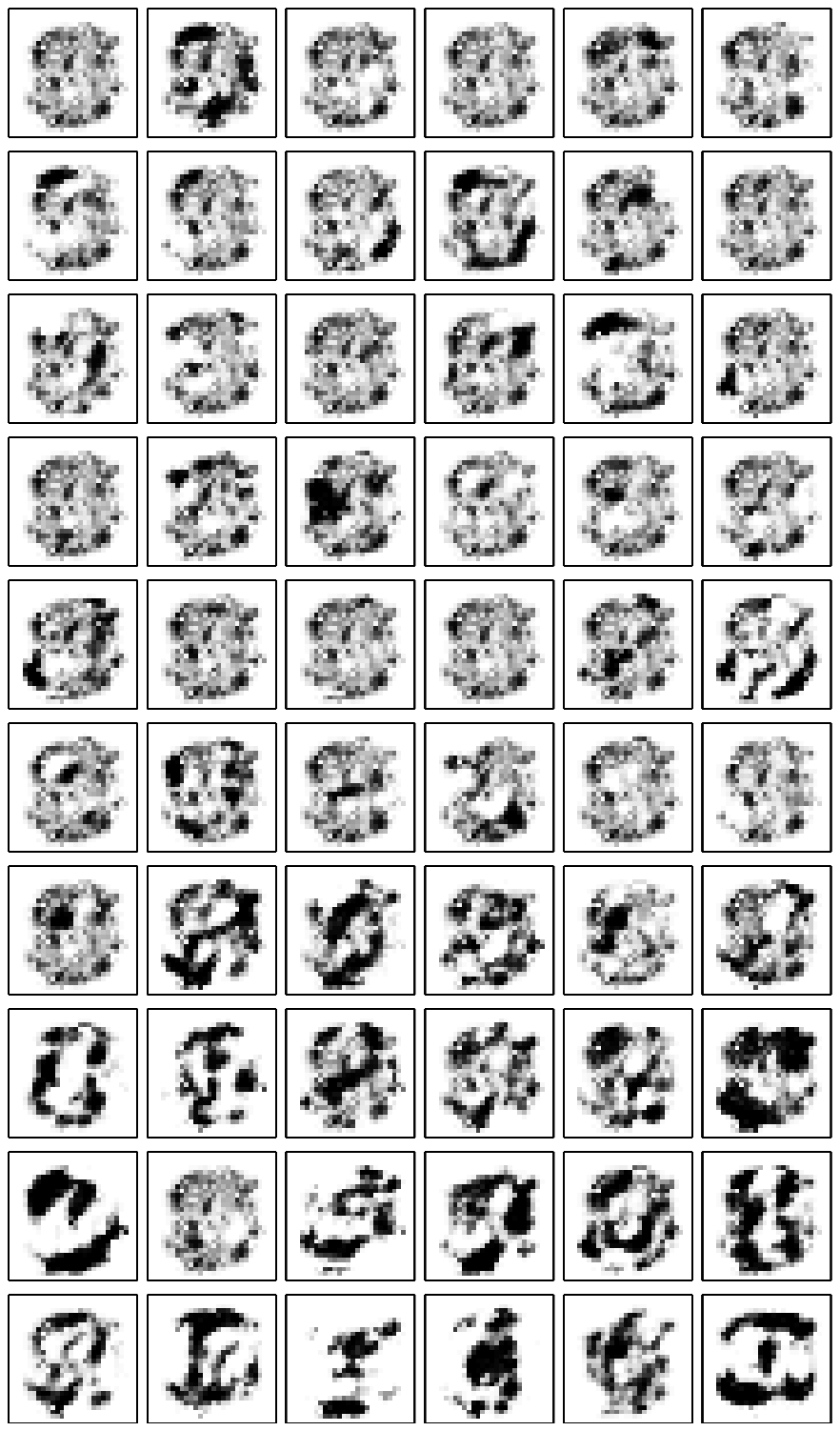}%
    \label{fig:mnist-features}%
  }\\%
  \subfloat[]{%
    \includegraphics[height=8cm,angle=90]{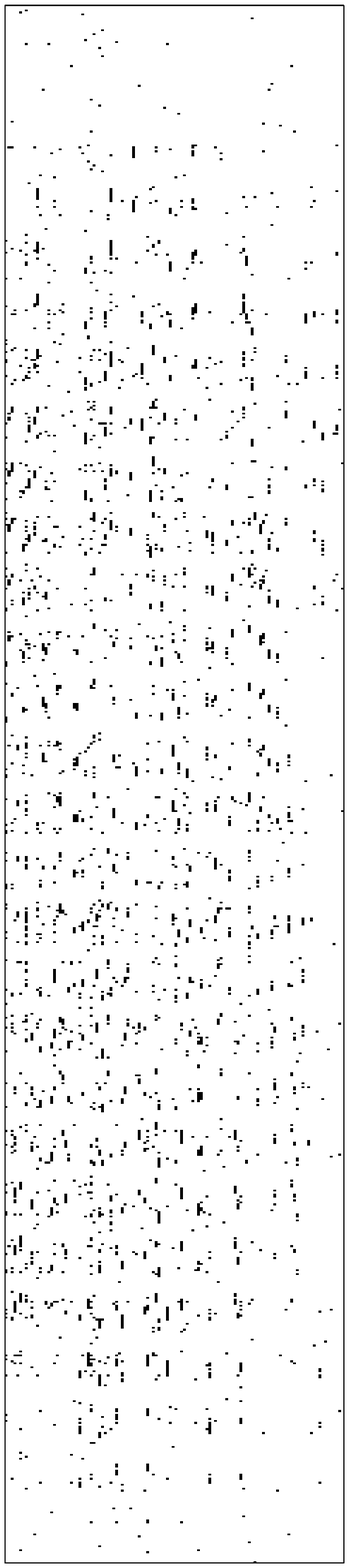}~%
    \includegraphics[height=1.15cm,angle=90]{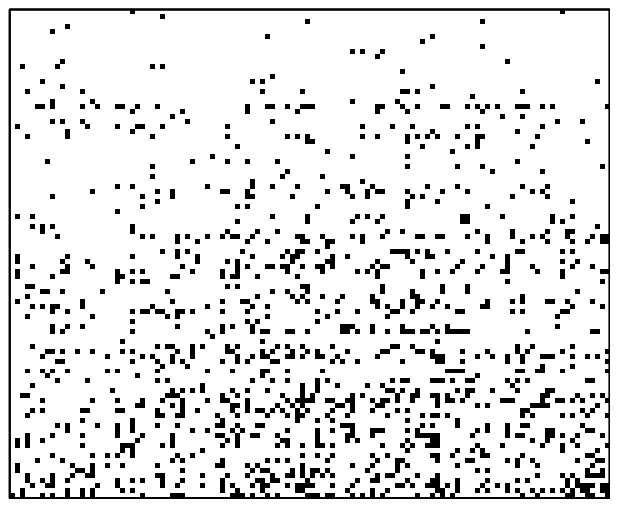}~%
    \includegraphics[height=0.65cm,angle=90]{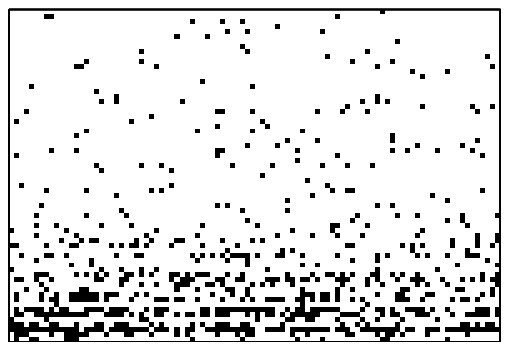}%
    \label{fig:mnist-matrices}%
  }
  \caption{\small MNIST Digits a)~Eight pairs of test reconstructions, with the
    bottom half of each digit missing.  The truth is the left image in each
    pair.  b)~120 features learned in the bottom layer, where black
    indicates that no edge exists.  c)~Activations in pixel space resulting
    from activating individual units in the deepest layer.  d)~Samples from
    the posterior of~$\bZ^{(0)}$,~$\bZ^{(1)}$ and~$\bZ^{(2)}$ (transposed).}
  \label{fig:mnist}
\end{figure}

\paragraph*{MNIST Digit Data}
We used a subset of the MNIST handwritten digit
data~\citep{lecun-etal-1998a} for training, consisting of 50~${28 \times
  28}$ examples of each of the ten digits.  We used an additional ten
examples of each digit for test data.  In this case, the lower-level
features are extremely sparse, as shown in Figure~\ref{fig:mnist-weights},
and the deeper units are simply activating sets of blobs at the pixel
level.  This is shown also by activating individual units at the deepest
layer, as shown in Figure~\ref{fig:mnist-features}.  Test reconstructions
are shown in Figure~\ref{fig:mnist-recons}.  A typical network had three
hidden layers, with approximately 120 in the first, 100 in the second and
70 in the third.  The binary matrices~$\bZ^{(0)}$,~$\bZ^{(1)}$,
and~$\bZ^{(2)}$ are shown in Figure~\ref{fig:mnist-matrices}.

\paragraph*{Frey Faces}
The Frey faces
data\footnote{\url{http://www.cs.toronto.edu/~roweis/data.html}} are
1965~${20\times 28}$ grayscale video frames of a single face with different
expressions.  We divided these into 1865 training data and 100 test data,
selected randomly.  While typical posterior samples of the network again
typically used three hidden layers, the networks for these data tended to
be much wider and more densely connected.  In the bottom layer, as shown in
Figure~\ref{fig:frey-weights}, a typical hidden unit would connect to many
pixels.  We attribute this to global correlation effects from every image
only coming from a single person.  Typical widths were 260 units, 120 units
in the second hidden layer, and 35 units in the deepest layer.

\begin{figure}[t!]
  \centering%
  \subfloat[]{%
    \includegraphics[height=8cm]{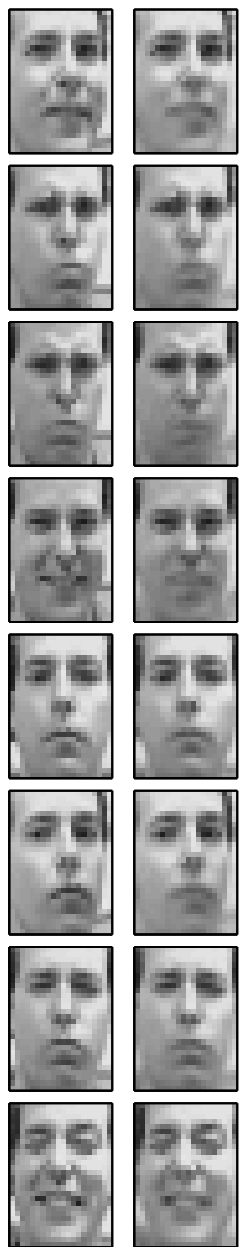}%
    \label{fig:frey-recons}%
  }~%
  \subfloat[]{%
    \includegraphics[height=8cm]{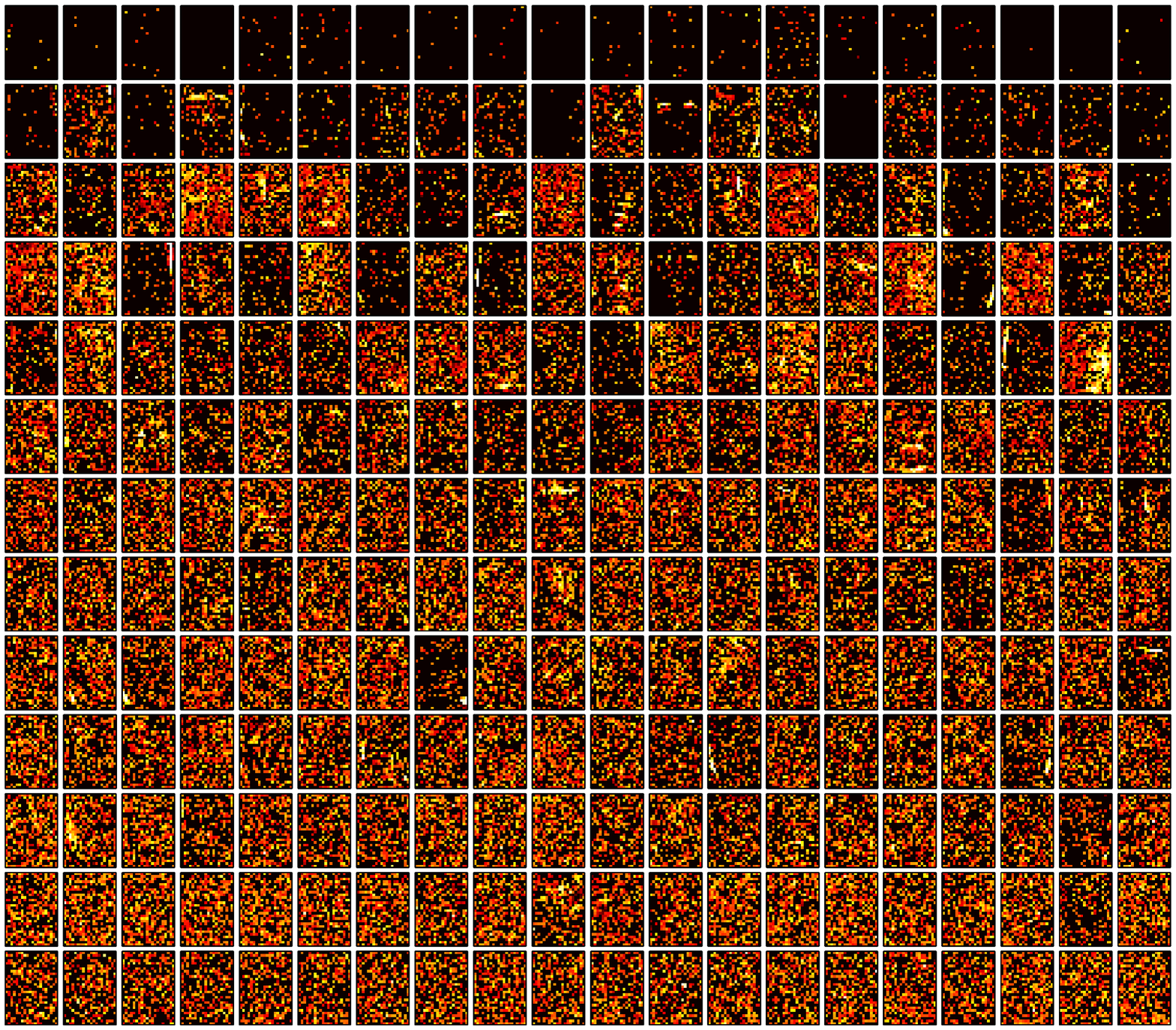}%
    \label{fig:frey-weights}%
  }%
  \caption{\small Frey faces a)~Eight pairs of test reconstructions, with the
    bottom half of each face missing.  The truth is the left image in each
    pair.  b)~260 features learned in the bottom layer, where black
    indicates that no edge exists.}
  \label{fig:freyface}
\end{figure}


In all three experiments, our MCMC sampler appeared to mix well and begins
to find reasonable reconstructions after a few hours of CPU time. It is
interesting to note that the learned sparse connection patterns
in~$\bZ^{(0)}$ varied from local (MNIST), through intermediate (Olivetti)
to global (Frey), despite identical hyperpriors on the IBP parameters.
This strongly suggests that flexible priors on structures are needed to
adequately capture the statistics of different data sets.

\section{Discussion}
This paper unites two areas of research---nonparametric Bayesian
methods and deep belief networks---to provide a novel nonparametric
perspective on the general problem of learning the structure of
directed deep belief networks with hidden units.

We addressed three outstanding issues that surround deep belief
networks. First, we allowed the units to have different operating
regimes and infer appropriate local representations that range from
discrete binary behavior to nonlinear continuous behavior. Second, we
provided a way for a deep belief network to contain an arbitrary
number of hidden units arranged in an arbitrary number of layers. This
structure enables the hidden units to have nontrivial joint
distributions. Third, we presented a method for inferring the
appropriate directed graph structure of deep belief network. To
address these issues, we introduced a novel \textit{cascading}
extension to the Indian buffet process---the cascading Indian buffet
process (CIBP)---and proved convergence properties that make it useful
as a Bayesian prior distribution for a sequence of infinite binary
matrices.

This work can be viewed as an infinite multilayer generalization of
the density network~\citep{mackay-1995a}, and also as part of a more
general literature of learning structure in probabilistic networks.
With a few exceptions
(e.g.,\ \citet{ramachandran-mooney-1998a,friedman-1998a, elidan-etal-2000a,beal-ghahramani-2006}),
most previous work on learning the structure of belief networks has
focused on the case where all units are observed
\citep{buntine-1991a,heckerman-etal-1995a,friedman-koller-2003a,koivisto-sood-2004a}.
The framework presented in this paper not only allows for an unbounded
number of hidden units, but fundamentally couples the model for the
number and behavior of the units with the nonparametric model for the
structure of the infinite directed graph. Rather than comparing
structures by evaluating marginal likelihoods of different models, our
nonparametric approach makes it possible to do inference in a single
model with an unbounded number of units and layers, thereby learning
effective model complexity. This approach is more appealing both
computationally and philosophically.

There are a variety of future research paths that can potentially stem
from the model we have presented here.  As we have presented it, we do
not expect that our MCMC-based unsupervised inference scheme will be
competitive on supervised tasks with extensively-tuned discriminative
models based on variants of maximum-likelihood learning.  However, we
believe that this model can inform choices for network depth, layer
size and edge structure in such networks and will inspire further
research into flexible nonparametric network models.


\subsection*{Acknowledgements}
The authors wish to thank Brendan Frey, David MacKay, Iain Murray and
Radford Neal for valuable discussions.  RPA is funded by the Canadian
Institute for Advanced Research.

\bibliographystyle{abbrvnat}
\bibliography{draft}

\appendix
\section{Proof of General CIBP Termination}
\label{sec:appendix}
In the main paper, we discussed that the cascading Indian buffet process
for fixed and finite~$\alpha$ and~$\beta$ eventually reaches a restaurant
in which the customers choose no dishes.  Every deeper restaurant also has
no dishes.  Here we show a more general result, for IBP parameters that
vary with depth, written~$\alpha^{(m)}$ and~$\beta^{(m)}$.

Let there be an inhomogeneous Markov chain~$\mcM$ with state
space~$\naturals$.  Let~$m$ index time and let the state at time~$m$ be
denoted~$K^{(m)}$.  The initial state~$K^{(0)}$ is finite.  The probability
mass function describing the transition distribution for~$\mcM$ at time~$m$
is given by
\begin{multline}
p(K^{(m+1)}=k\given K^{(m)},\alpha^{(m)},\beta^{(m)})
  =\\
 \frac{1}{k!}\exp\left\{
    -\alpha^{(m)}\sum^{K^{(m)}}_{k'=1}\frac{\beta^{(m)}}{k'+\beta^{(m)}-1}
  \right\}
  \left(\alpha^{(m)}\sum^{K^{(m)}}_{k'=1}\frac{\beta^{(m)}}{k'+\beta^{(m)}-1}
  \right).
\end{multline}
    
\begin{theorem}
  If there exists some~${\baralpha<\infty}$ and~${\barbeta<\infty}$ such
  that~$\forall m$, ${\alpha^{(m)}<\baralpha}$
  and~${\beta^{(m)}<\barbeta}$, then~${\lim_{m\to\infty} p(K^{(m)}=0)
  = 1}$.
\end{theorem}

\begin{proof}
  Let~$\naturals^{+}$ be the positive integers.  The~$\naturals^{+}$ are a
  communicating class for the Markov chain (it is possible to eventually
  reach any member of the class from any other member) and
  each~$K^{(m)}\in\naturals^{+}$ has a nonzero probability of transitioning
  to the absorbing state~${K^{(m+1)}=0}$, i.e.~${p(K^{(m+1)}=0\given
    K^{(m)}) > 0}$, $\forall K^{(m)}$.  If, conditioned on nonabsorption,
  the Markov chain has a stationary distribution (is
  \textit{quasi-stationary}), then it reaches absorption in finite time
  with probability one.  Heuristically, this is the requirement that,
  conditioned on having not yet reached a restaurant with no dishes, the
  number of dishes in deeper restaurants will not explode.
  
  The quasi-stationary condition can be met by showing that~$\naturals^{+}$
  are positive recurrent states.  We use the Foster--Lyapunov stability
  criterion (FLSC) to show positive-recurrency of~$\naturals^{+}$.  The FLSC
  is met if there exists some
  function~${\mcL(\cdot):\naturals^{+}\to\reals^{+}}$ such that for
  some~$\epsilon>0$ and some finite~$B\in\naturals^{+}$,
  \begin{align}
    \sum^{\infty}_{k=1}p(K^{(m+1)}=k\given
    K^{(m)})\,\left(\mcL(k)-\mcL(K^{(m)})\right)
    & < -\epsilon
    \text{ for } K^{(m)} > B\\
    \sum^{\infty}_{k=1}p(K^{(m+1)}=k\given K^{(m)})\,\mcL(k) 
    & <\infty
    \text{ for } K^{(m)} \leq B.
  \end{align}
  For Lyapunov function~$\mcL(k)=k$, the first condition is
  equivalent to
  \begin{align}
    \left(\alpha^{(m)}\sum^{K^{(m)}}_{k=1}\frac{\beta^{(m)}}{k+\beta^{(m)}-1}\right)
    -K^{(m)} < -\epsilon.
  \end{align}
  We observe that
  \begin{align}
    \alpha^{(m)}\sum^{K^{(m)}}_{k=1}\frac{\beta^{(m)}}{k+\beta^{(m)}-1}
    &<
    \baralpha\sum^{K^{(m)}}_{k=1}\frac{\barbeta}{k+\barbeta-1},
  \end{align}
  for all~$K^{(m)}>0$.  Thus, the first condition is satisfied for any~$B$
  that satisfies the condition for~$\baralpha$ and~$\barbeta$.  That such
  a~$B$ exists for any finite~$\baralpha$ and~$\barbeta$ can be seen by the
  equivalent condition
  \begin{align}
    \left(\baralpha\sum^{K^{(m)}}_{k=1}\frac{\barbeta}{k+\barbeta-1}\right)
    -K^{(m)} < -\epsilon \text{ for } K^{(m)} > B.
  \end{align}
  As the first term is roughly logarithmic in~$K^{(m)}$, there exists
  some finite~$B$ that satisfies this inequality.  The second FLSC
  condition is trivially satisfied by the observation that Poisson
  distributions have a finite mean.  
\end{proof}
\end{document}